\documentclass[nohyperref]{article}

\usepackage{microtype}
\usepackage{graphicx}
\usepackage{subfigure}
\usepackage{booktabs} % for professional tables

\usepackage{hyperref}
\usepackage[accepted]{icml2023_arxiv}

\usepackage{amsmath}
\usepackage{amssymb}
\usepackage{mathtools}
\usepackage{amsthm}
\usepackage{thmtools, thm-restate}

\usepackage[capitalize,noabbrev]{cleveref}

\theoremstyle{plain}
\newtheorem{theorem}{Theorem}[section]
\newtheorem{proposition}[theorem]{Proposition}
\newtheorem{lemma}[theorem]{Lemma}
\newtheorem{corollary}[theorem]{Corollary}
\theoremstyle{definition}
\newtheorem{definition}[theorem]{Definition}

\theoremstyle{remark}

\usepackage[textsize=tiny]{todonotes}

\usepackage{amsmath}
\usepackage{amsfonts}
\usepackage{dsfont}
\usepackage{amsthm}
\usepackage{color}
\usepackage{thmtools}

\newcommand{\cS}{\mathcal{S}}
\newcommand{\ind}{\mathds{1}}
\newcommand{\bag}{\mathcal{B}}
\newcommand{\cX}{\mathcal{X}}
\newcommand{\cY}{\mathcal{Y}}
\newcommand{\popl}{\mathcal{L}}
\newcommand{\rad}{\mathfrak{R}}
\newcommand{\hyps}{\mathcal{H}}
\newcommand{\cD}{\mathcal{D}}
\newcommand{\cF}{\mathcal{F}}
\newcommand{\Rset}{\mathbb{R}}

\newcommand{\bernoulli}{\text{Bernoulli}}
\newcommand{\hy}{\widehat{y}}
\newcommand{\tg}{\widetilde{g}}
\newcommand{\tL}{\widetilde{\ell}}
\newcommand{\ty}{\widetilde{y}}
\newcommand{\tY}{\widetilde{Y}}
\newcommand{\bw}{\mathbf{w}}
\newcommand{\cW}{\mathcal{W}}
\newcommand{\inner}[1]{\langle #1 \rangle}
\newcommand{\easyllp}{\textsc{{\hyphenpenalty=10000EasyLLP}}}

\DeclareMathOperator*{\E}{\mathbb{E}}
\DeclareMathOperator*{\PP}{\mathbb{P}}

\newcommand \cov {\operatorname*{Cov}}

\icmltitlerunning{Easy Learning from Label Proportions}

\begin{document}

\twocolumn[
\icmltitle{Easy Learning from Label Proportions}



\icmlsetsymbol{equal}{*}

\begin{icmlauthorlist}
\icmlauthor{Robert Istvan Busa-Fekete}{yyy}
\icmlauthor{Heejin Choi}{yyy}
\icmlauthor{Travis Dick}{yyy}
\icmlauthor{Claudio Gentile}{yyy}
\icmlauthor{Andres Munoz medina}{yyy}
\end{icmlauthorlist}

\icmlaffiliation{yyy}{Google Research, USA}

\icmlcorrespondingauthor{Robert Istvan Busa-Fekete}{busarobi@google.com}
\icmlcorrespondingauthor{Heejin Choi}{heejincs@gmail.com}
\icmlcorrespondingauthor{Travis Dick}{tdick@google.com}
\icmlcorrespondingauthor{Claudio Gentile}{cgentile@google.com}
\icmlcorrespondingauthor{Andres Munoz medina}{ammedina@google.com}

\icmlkeywords{Machine Learning, ICML}

\vskip 0.3in
]

\printAffiliationsAndNotice{\icmlEqualContribution} % otherwise use the standard text.

\begin{abstract}
We consider the problem of Learning from Label Proportions (LLP), a weakly supervised classification setup where instances are grouped into ``bags’’, and only the frequency of class labels at each bag is available. Albeit, the objective of the learner is to achieve low task loss at an individual instance level. Here we propose \easyllp: a flexible and simple-to-implement debiasing approach based on aggregate labels, which operates on arbitrary loss functions.
%information and label prior probabilities. 
Our technique allows us to accurately estimate the expected loss of an arbitrary model at an individual level. We showcase the flexibility of our approach by applying it to
%We show how to incorporate this unbiased loss 
popular learning frameworks, like Empirical Risk Minimization (ERM) and Stochastic Gradient Descent (SGD) with provable guarantees on instance level performance. More concretely, we exhibit a
%show that using an elegant 
variance reduction technique that makes the quality of LLP learning deteriorate only by a factor of $k$ ($k$ being bag size) in both ERM and SGD setups, as compared to full supervision. Finally, we validate our theoretical results on multiple datasets demonstrating our algorithm performs as well or better than previous LLP approaches in spite of its simplicity. 
\end{abstract}

\section{Introduction}\label{sec:intro}
In traditional supervised learning problems, a learner has access to a sample of labeled examples. This collection of labeled examples is used to fit a model (decision trees, neural networks, random forests, ...) by minimizing a loss over the observed sample. By contrast, in the problem of Learning from Label Proportions (LLP), the learner only observes collections of unlabeled feature vectors called {\em bags}, together with the proportion of positive examples in that bag (Figure \ref{f:llp}).
The LLP problem is motivated by a number of applications where access to individual examples is too expensive or impossible to achieve, or available at aggregate level for privacy-preserving
reasons. Examples include e-commerce, fraud detection, medical databases \cite{pat+14}, high energy physics \cite{dery+18}, election prediction \cite{sun+17}, medical image analysis \cite{bor+18}, remote sensing \cite{ding+17}.

\begin{figure}[t]
\centering
\includegraphics[width=0.35\textwidth]{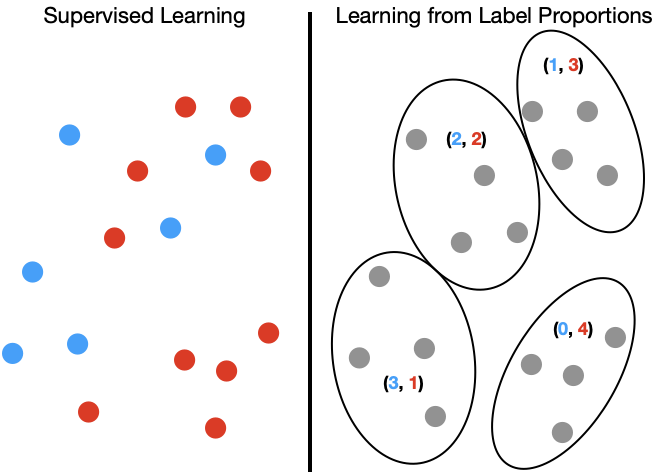}
\caption{Comparison of supervised learning and LLP. In supervised learning the learner observes each individual label red or blue. In LLP, the learner observes bags of unlabeled examples and the total number of blue and red labels in each bag.\label{f:llp}}
\end{figure}

As a weakly supervised learning paradigm, LLP  traces back to at least \cite{dfk05,mco07,QuadriantoSCL08, rue10,yu+13}, and was motivated there by learning scenarios where access to individual examples is often not available. A paradigmatic example is perhaps a political campaign trying to predict the preference of the electorate. Since voting is anonymous, a political analyst may not be able to observe individual votes, yet, they have access to aggregate voting preferences at the district level.

The problem has received renewed interest more recently~(e.g., \cite{dulac2019deep,lu+19,ScottZ20,NEURIPS2021_33b3214d,52071,lu+21,z+22}), driven by the desire to provide more privacy to user information. For instance the ad conversion reporting system proposed by Apple, SKAN, allows an ad tech provider to receive conversion information only aggregated across multiple impressions. This aggregation is intended to obfuscate individual activity. A similar API has also been proposed by Google Chrome and Android to report aggregate conversion information (e.g., \cite{privacy_sandbox}). Given the importance of conversion modeling for online advertising, learning how to train a model using only these aggregates has become a crucial task for many data-intensive online businesses. 

Research in LLP can be coarsely divided into two types of goals: Learning a bag classifier that correctly predicts label proportions, and learning an individual classifier that can correctly predict instance labels. The former has been the focus of most of the literature in this area. Representative papers include \cite{dfk05,mco07,yu+13}.
%\cite{sanjeev} and it may be a good fit to problems like the voting preference scenario previously described. 
Nevertheless, it is known~\cite{yu+13,ScottZ20} that a good bag level predictor can result in a very bad instance level predictor. Finding a good instance level predictor using only label proportions is a harder problem, and solutions introduced so far require either making some assumptions on the data generation process \cite{ScottZ20,z+22} or on the model class \cite{QuadriantoSCL08}. Other solutions involve solving complex combinatorial problems \cite{dulac2019deep} or require that an example belongs to multiple bags (e.g., \cite{NEURIPS2021_33b3214d,52071}).

Our work falls in the category of learners that can produce event level predictions. Unlike previous research, we provide an extremely simple, yet powerful debiasing method which is applicable to the standard learning scenario of i.i.d. samples. Crucially, our proposed algorithm can be applied to virtually any model class against any loss function. We elucidate the flexibility of our approach by applying it to two widely interesting algorithmic techniques, Empirical Risk Minimization (ERM) and
Stochastic Gradient Descent (SGD), emphasize the theoretical underpinning of the resulting LLP methods, and complement our findings with an extensive experimental investigation on benchmark datasets.

\paragraph{Main contributions.} The contribution of our paper can be summarized as follow.
%Our paper has four important contributions to this problem. %
\begin{enumerate}
\item We provide a general debiasing technique for estimating the expected instance loss (or loss gradient) of an arbitrary model using only label proportions.

\item We provide a reduction of ERM with label proportions to that of ERM with individual labels, and show that when the learner observes bags of size $k$, the sample complexity of ERM with label proportions increases only by a factor of $k$.

\item Likewise, we provide an analysis of SGD using only label proportions and show that for bags of size $k$, its regret increases by only a factor of $k$.

%\item We provide an open source implementation of our algorithm that can be easily adapted to any Keras based learner.
\item We carry out an extensive set of experiments comparing our LLP methods to known methods available in the LLP literature, and show that on the tested datasets our methods perform as well as or better than the competitors when evaluated at the instance level.
\end{enumerate}

\paragraph{Related work.}
Interest in LLP traces back to at least \cite{dfk05,mco07,QuadriantoSCL08, rue10,10.1007/978-3-642-23808-6_23,yu+13,pat+14}. The literature in recent years has become quite voluminous, so we can hardly do justice of it here. In what follows, we comment on contrast to the references from which we learned about LLP problems.

In \cite{dfk05} the authors consider a hierarchical model that generates labels according to the given label proportions and proposed an MCMC-based inference scheme which, however, is not scalable to large training sets. \citet{mco07} show how standard supervised learning algorithms (like SVM and $k$-Nearest Neighbors) can be adapted to LLP by a reformulation of their objective function. In \cite{QuadriantoSCL08}, the authors propose a theoretically-grounded way of estimating the mean of each class through the sample average of each bag, along with the associated label proportion. The authors make similar assumptions to ours, in that the class-conditional distribution of data is independent
of the bags. Yet, their estimators rely on very strong  assumptions, like conditional exponential models, which are not a good fit to nowadays Deep Neural Network (DNN) models. Similar limitations are contained in \cite{pat+14}.
\citet{rue10} proposes an adaptation of SVM to the LLP setting, but this turns out to be restricted to linear models in some feature space. Similar limitations are in the $\alpha$-SVM method proposed by \citet{yu+13}, the the non-parallel SVM formulation of \citet{qi+17}, and the pinball loss SVM in \cite{shi+19}.
The original $\alpha$-SVM formulation was extended to other classifiers; e.g., \citet{lt15} extend the formulation to CNNs with a generative model whose associated maximum likelihood estimator is computed via Expectation Maximization, which turns out not to be scalable to sizeable DNN architectures. \citet{10.1007/978-3-642-23808-6_23} propose a method based on $k$-means to identify a clustering of the data which is compatible with the label proportions, but their method suffers from an extremely high computational complexity.  

Many of these papers are in fact purely experimental in nature, and their main goal is to adapt the standard supervised learning formulation to an LLP formulation so as to obtain a bag level predictor. 

On the learning theory side, besides the already mentioned \cite{QuadriantoSCL08,pat+14}, are the efforts contained in \cite{NEURIPS2021_33b3214d,52071}, the task of learning from multiple unlabeled datasets considered in \cite{lu+19,lu+21}, and the statistical learning agenda pursued in \cite{ScottZ20, z+22} (and references therein from the same authors). In \cite{NEURIPS2021_33b3214d,52071} the author is essentially restricting to linear-threshold functions and heavily relies on the fact that an example can be part of multiple bags, while we are working with non-overlapping i.i.d. bags and general model classes. In \cite{lu+19,lu+21} the authors consider a problem akin to LLP. Similar to our paper, the authors generate surrogate labels and propose a debiasing procedure via linear transformations. Yet, the way the solve the debiasing problem forces them to impose further restrictions on the bags, like the separation of the class prior distributions across different bags. The bags proposed in our setup are drawn i.i.d. from the same distribution, which is a scenario where many of these algorithms would fail. Moreover, the convergence results to the event level performance are only proven with families of loss function (e.g., proper loss functions).
\citet{ScottZ20} introduced a principled approach to LLP based on a reduction to learning with label noise. As in \cite{lu+19}, their basic strategy is to pair bags, and view each pair as a task of learning with label noise, where label proportions are related to label flipping probabilities.
The authors also established generalization error guarantees at the event level, as we do here.
\cite{z+22} extend their results to LLP for multiclass classification.
From a technical standpoint, these two papers have similar limitations as \cite{lu+19,lu+21}. Besides, the risk measure they focus on is balanced risk rather than classification risk, as we do here.

In our experiments (Section \ref{sec:experiments}), we empirically compare to the MeanMap method from \cite{QuadriantoSCL08}, and to a label generation approach from \cite{dulac2019deep}, the latter viewed as representative of recent applications of DNNs to LLP.

\section{Setup and Notation}\label{sec:setup}
Let $\cX$ denote a feature (or instance) space and $\cY = \{0, 1\}$ be a binary label space. We assume the existence of a joint distribution $\cD$ on $\cX \times \cY$, and let $p = \PP_{(x,y) \sim \cD}(y = 1)$ denote the probability of drawing a sample $(x,y) \in \cX \times \cY$ from $\cD$ with label $y = 1$. For a natural number $n$, let $[n] = \{i \in \mathbb{N} \colon i \leq n\}$.

A labeled {\em bag} of size $k$ is a sample $\bag = \{x_1,\ldots,x_k\}$, together with the associated {\em label proportion} $\alpha(\bag) = \frac{1}{k} \sum_{j=1}^k y_{j}$, where $(x_1,y_1), \ldots, (x_k,y_k)$ are drawn i.i.d. according to $\cD$.
We assume the learner has access to a collection $\mathcal{S} = \{ (\bag_i, \alpha_i),\,i \in [n]\}$ of $n$ labeled bags of size $k$, where $\bag_i = \{x_{ij} \colon j \in [k]\}$, $\alpha_i = \alpha(\bag_i) = \frac{1}{k} \sum_{j=1}^k y_{ij}$ is the label proportion of the $i$-th bag, and all the involved samples $(x_{ij},y_{ij})$ are drawn i.i.d. from $\cD$. In words, the learner receives information about the $nk$ labels $y_{ij}$ of the $nk$ instances $x_{ij}$ only in the aggregate form determined by the $n$ label proportions $\alpha_i$ associated with the $n$ labeled bags $(\bag_i,\alpha_i)$ in collection $\mathcal{S}$. Notice, however, that the instances $x_{ij}$ are individually observed.

Given a hypothesis set $\hyps$ of functions $h$ mapping $\cX$ to a prediction space $\widehat{\cY}$, and a loss function $\ell~\colon \widehat{\cY}~\times~\cY~\to~\Rset$, the learner receives a collection $\mathcal{S}$, and tries to find a hypothesis $h \in \hyps$ with the smallest \emph{population loss} (or {\em risk})
%\begin{equation*}
\(
  \popl(h) = \E_{(x,y) \sim\cD}[\ell(h(x), y) ]
\)
%\end{equation*}
with high probability over the random draw of $\mathcal{S}$.
When clear from the surrounding context, we will omit subscripts like “$(x,y)\sim\mathcal{D}$" or “$\mathcal{D}$" from probabilities and expectations.

We shall consider two broadly used learning methods for solving the above learning problem, Empirical Risk Minimization (ERM, Section \ref{s:erm}), or regularized versions thereof, and Stochastic Gradient Descent (SGD, Section~\ref{s:sgd}).
In this latter context, we will consider a parameter space $\cW$ and consider a learner that tries to optimize a loss $\ell \colon \cW \times \cX \times \cY \to \Rset$ iteratively over a collection of bags.

\section{Surrogate labels}\label{sec:surrogate}
%
%In this section 
We now introduce the main tool for our algorithms for learning with label proportions: surrogate labels. The objective of using surrogate labels is to easily transform a label proportion problem into a traditional learning scenario with one label per example. While the idea of surrogate labels is natural, we show that na\"ively using them can result in very unstable learning algorithms. Our main contribution is thus to understand how to reduce this instability.
%two heuristics to generate them. We finally formalize these heuristics through the use of a \emph{link function}. 
%
\begin{definition}
\label{def:surrogate}
Let $(\bag,\alpha)$ be a labeled bag,
%let $\alpha$ denote the label proportion of $\bag$ 
and let $x \in \bag$. A surrogate label $\ty \in \{0,1\}$ for $x$ is a sample from the distribution $\bernoulli(\alpha)$.
\end{definition}

\subsection{Learning with surrogate labels}\label{ss:erm}

By using surrogate labels we may generate a new sample $\tilde{\cS}$ of individually label examples: $(x_{ij}, \ty_{ij})$ where $\ty_{ij} \sim \bernoulli(\alpha_i)$ is drawn independently for each $j \in [k]$. 

At first impression, the surrogate labels may seem as a poor replacement of the true labels. After all, for a bag of size $k$, the probability that a surrogate label $\ty_{ij}$ matches $y_{ij}$ can be as low as $\frac{1}{k}$. The following proposition shows that evaluating a function using surrogate labels is indeed a bad proxy for using true labels. However, we shall see that a simple linear correction can yield a much better approximation.\footnote
{
All omitted  proofs are contained in the appendix.
}

\begin{restatable}{proposition}{propBias}
\label{prop:bias}
Given a sample $(x_1,y_1),\ldots, (x_k,y_k)$ drawn i.i.d. according to $\cD$, let $(\bag, \alpha)$ be the corresponding labeled bag of size $k$, for some $k \geq 1$. Let $\ty_1, \ldots, \ty_k$ denote the surrogate labels sampled according to Definition~\ref{def:surrogate}.
Let $g\colon \cX \times \cY \to \Rset^d$ be any (measurable) function, for some output dimension $d \geq 1$. Then for all $j$ it holds that
\begin{align}
    \E[g(x_j, \ty_j)] &= \frac{1}{k} \E[g(x_j, y_j)] + \frac{(k\!\!-\!\!1)(1\!\!-\!\!p)}{k} \E[g(x_j,0)]  \notag\\
    &\qquad\qquad + \frac{(k-1)p}{k}\E[g(x_j, 1)]\,, \label{tl:surrogate}
\end{align}
where the expectation on the left-hand side is taken over the original sample $(x_1,y_1),\ldots, (x_k,y_k)$, and the randomness in the generation of the surrogate label $\ty_j$ (which is a random function of $\alpha = \frac{1}{k}\sum_{j=1}^k y_j$). 
\end{restatable}
The above proposition shows that we can easily obtain an unbiased estimate of the expectation of any function $g$ by applying a simple linear transformation to the output of $g(x_j, \ty_j)$.
\begin{restatable}{corollary}{corDebiasedLoss}
\label{cor:debiasedLoss}
With the notation of Proposition~\ref{prop:bias}, for any function $g \colon \cX \times \cY \to \Rset^d$ define $\tg \colon \cX \times \cY \to \Rset^d$ as:
\begin{align*}
    \tg(x, \ty) &=
    k g(x, \ty)
    - (k\!\!-\!\!1)(1\!\!-\!\!p) g(x, 0) - (k\!\!-\!\!1)p g(x, 1)\,,
\end{align*}
Then for an element $x \in \bag$ we have
\begin{equation*}
    \E[\tg(x, \ty)] = \E_{(x,y)\sim \cD}[g(x, y)]\,,
\end{equation*}
where the expectation on the left is taken over $(\bag, \alpha)$ and the choice of the surrogate label $\ty$.
\end{restatable}
We would like to highlight the importance of this corollary. While there has been a lot of research in LLP; to the best of our knowledge this is the first expression that shows that one can recover an unbiased estimate of an arbitrary function $g$ using only information from label proportions. 

While the above corollary provides us with a straightforward way to estimate the expectation of a function $g$ (which can be for instance specialized to a loss function), note that the variance of $\tg$ increases as the number of elements in each bag grows. Indeed, since all terms in the definition of $\tg$ have a factor of $k$, we can expect the variance of the estimator to grow as $k^2$. This variance may be prohibitively high even for moderate values of $k$. This is illustrated in \Cref{fig:variance_comparison} on a simple case. We empirically compare the variance of four estimates of $\E[g(x,y)]$. The first two  are based on surrogate labels, either using one surrogate sample per bag (``surrogate one") or $k$ surrogate samples per bag (``surrogate avg."). The remaining two use soft label corrected samples (see next section), either based on a single sample (``derandomized one") or $k$ samples (``derandomized avg.").

In order to make our estimators useful for learning, we need to figure out a way to reduce their variance. We next show that by averaging the estimator within a bag and using a 
simple derandomization of the surrogate labels, we can drastically reduce the variance by a factor of $k$.

\subsection{Derandomized Loss Estimates}
In this section we show that by replacing the surrogate label by its expectation we can reduce the variance of an estimator $\tg$. We begin by the simple observation (which follows from a simple case analysis and easy algebraic manipulations) that $\tg$ from Corollary \ref{cor:debiasedLoss} can equivalently be expressed as
\begin{align}
    \tg(x, \ty) &= \left(k (\ty - p) + p\right)g(x, 1) \nonumber\\
    & \qquad + \left(k (p - \ty) + (1 - p)\right) g(x, 0).\label{eq:tlsimple}
\end{align}
Expressed in this form, we see that the dependence on the surrogate label $\ty$ is linear. Thus taking expectation with respect with the choice of the surrogate label should be fairly simple.
\begin{definition}
With the notation of Proposition~\ref{prop:bias}, we overload the definition of $\tg$ and define a \emph{soft label corrected function} $\tg \colon \cX \times [0,1] \to \Rset^d$ for example $x_j \in \bag$  as
\begin{align}
\tg(x_j, \alpha) &= \E[\tg(x_j, \ty_j) | \bag, \alpha] \nonumber \\
&=  \left(k(\alpha - p) + p\right) g(x_j, 1) \nonumber \\
 & \qquad + \left(k(p - \alpha) + (1- p)\right) g(x_j, 0).
     \label{eq:simple_tlalpha}
\end{align}
\end{definition}
It is important to notice that even if the function $g$ is defined over $\cX \times \cY$, the soft label corrected function $\tg$ is properly defined over $\cX \times [0,1]$. 
The following lemma shows that, for any underlying function $g$, the soft label corrected function $\tg$ remains an unbiased estimate for the expectation of $g$. More importantly, we also demonstrate that the variance of the estimator based on the soft label corrected function
%this loss function 
is always smaller than the variance of the one based on 
%loss induced by using simple 
surrogate labels. 
\begin{lemma}\label{l:derand}
    Let $(\bag, \alpha)$ be a labeled bag of size $k$ sampled from $\cD$. Let $g\colon \cX \times \cY \to \Rset^d$ be  an arbitrary function and $\tg$ denote its corresponding soft label corrected function. Then for $x_j \in \bag$ we have
    \[
    \E_{\bag, \alpha}[\tg(x_j, \alpha)] = \E_{(x, y) \sim \cD}[ g(x, y) ]~.
    \]
    Moreover, if $\bag = \{x_1, \dots, x_k\}$, and $\ty_1, \dots, \ty_k \sim \bernoulli(\alpha)$ are i.i.d. surrogate labels then
    \[
    \E\left[
        \Bigl \|\frac{1}{k} \sum_{j=1}^k \tg(x_j, \ty_j) \Bigl\|^2
    \right]
    \geq
    \E\left[
        \Bigl \|\frac{1}{k} \sum_{j=1}^k \tg(x_j, \alpha) \Bigl\|^2
    \right]~.
    \]
\end{lemma}

\begin{proof}
From Corollary~\ref{cor:debiasedLoss}, the definition of the soft label corrected function and the properties of conditional expectation we have, for $x_j \in \bag$:
\begin{align*}
    \E_{(x,y) \sim \cD}[g(x,y)] = \E_{\bag, \alpha}[\tg(x_j, \ty_j)]
    &= \E_{\bag, \alpha}\left[\E[\tg(x_j, \ty_j) | \bag, \alpha]\right]\\
    &= \E_{\bag, \alpha}[\tg(x_j, \alpha)].
\end{align*}
To prove the norm inequality, let  $U = \sum_{j=1}^k \tg(x_j, \ty_j)$. Note that $\langle U, U \rangle = \left\|\sum_{j=1}^k \tg(x_j, \ty_j)\right\|^2$. By manipulating the expectation of this inner product and the properties of the conditional expectation we have:
    \begin{align*}
        \E\left[\langle U, U\rangle\right] = \E\left[\E\left[\langle U, U \rangle | \bag, \alpha\right] - \left\|\E[ U  | \bag, \alpha]\right\|^2 \right] \\
        + \left\|\E[ U  | \bag, \alpha]\right\|^2~.
    \end{align*}
    Now, notice that 
    $$\E\left[\langle U, U \rangle | \bag, \alpha\right] - \left\|\E[ U  | \bag, \alpha]\right\|^2  = \text{tr}(\cov[U | \bag, \alpha]) \geq 0$$
    since the covariance matrix of a vector-valued random variable is positive semi-definite. By using the definition of $U$ and dividing both sides by $k^2$ gives the claimed result.
\iffalse    
    it then follows that
    \[
    \E\left[
        \left \| \sum_{j=1}^k \tg(x_j, \ty_j) \right\|^2
    \right]
    \geq
    \E\left[
        \left \| \sum_{j=1}^k \tg(x_j, \alpha) \right\|^2~.
    \right]
    \]
    The result is proven by multiplying both sides by $\frac{1}{k^2}$.
\fi
\end{proof}

The above lemma shows that using the soft label corrected function indeed reduces the variance of our estimator. However, \eqref{eq:simple_tlalpha} still has a dependence on $k$ which at first should make variance increase like $\Omega(k^2)$. Notice however that because $\alpha = \frac{1}{k}\sum_{j=1}^k y_j$, and each $y_j \sim \bernoulli(p)$,  we expect by standard concentration arguments that $k(\alpha - p) \in O(\sqrt{k})$ which should imply that the variance scales like $k$. This would be a significant variance reduction compared to the use of surrogate labels. 
The following theorem shows that indeed, the variance of these estimates is asymptotically $k$ and not $k^2$.

\begin{figure}
    \centering
    \includegraphics[width=0.55\columnwidth]{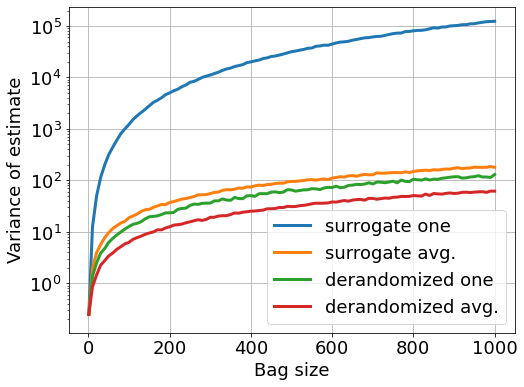}
    \caption{Comparison of the variance of LLP estimates using either surrogate labels or derandomization, and either averaging over the bag or using a single estimate per bag.
    This plot was generated using $g(x,y) = \ind\{x \leq 0.5\} y$ and with a data distribution $\cD$ that samples $x$ uniformly from $[0,1]$ and sets $y = \ind\{x \leq 0.5\}$.
    }
    \label{fig:variance_comparison}
\end{figure}

\begin{restatable}{theorem}{thmGradientSquared}
\label{t:gradient_squared}
Let $g \colon \cX \times \cY \to \Rset^d$ be such that $\sup_{x,y} \|g(x,y)\|^2\leq M$, and denote by $\tg$ its corresponding soft labeled corrected function. 
For each $j \in [k]$, let 
$\tilde g_j = \tilde g(x_j,\alpha)$.
Then, for any size $k \geq 1$ and any $j \in [k]$,
\begin{align}
\E\left[\Bigl|\Bigl|\frac{1}{k} \sum_{i=1}^k\tilde g_i\Bigl|\Bigl|^2\right] \leq \E[||\tilde g_j||^2]~.\label{eq:grad}
\end{align}
Moreover, denoting for brevity $g_0 = g(x,0)$ and $g_1 = g(x,1)$, there exist universal constants $C_1, C_2$ such that 
\begin{align*}
\E[||\tilde g_j||^2] &\leq C_1 + kp(1-p) \E\Bigl[||g_0-g_1||^2 \Bigl]~, \\
\E\left[\Bigl|\Bigl|\frac{1}{k} \sum_{i=1}^k\tilde g_i\Bigl|\Bigl|^2\right] 
&\leq C_2 +  kp(1-p)\Bigl|\Bigl|\E[g_0-g_1]\,\Bigl|\Bigl|^2~,
\end{align*}
where $p = \PP_{(x,y) \sim \cD}(y = 1)$.
\end{restatable}
The bound in the above theorem confirms our intuition. Moreover, it shows that the variance grows slower for datasets where $p$ is close to $1$ or $0$. This is intuitively clear, for very skewed datasets, we expect label proportions to provide a better description of the true labels. In the extreme cases where $p = 0$ or $p=1$, LLP becomes equivalent to learning from individual examples. 

The results of this section have demonstrated that for any function $g$, one can obtain an estimator of its expectation using only label proportions. More importantly the variance of this estimator only scales linearly with the bag size.

\paragraph{Note about knowledge of population level positive rate.} At this point the reader is aware that the definition of the soft label corrected function requires knowledge of the population level positive rate $p$. While the exact value of $p$ is unknown, one can easily estimate it from the label proportions itself. Indeed, using the fact that the generated bags are i.i.d. it is easy to see that 
%\begin{equation*}
\(
    \widehat{p} = \frac{1}{n} \sum_{i=1}^n \alpha_i
    = \frac{1}{nk}\sum_{i,j}y_{ij}
\)
%\end{equation*}
is a very good estimator for $p$. 

\paragraph{\easyllp.} We now have all elements to introduce the \easyllp\ framework for learning from label proportions. The framework consists of specializing the function $g$ for particular learning tasks. Two notable instantiations of \easyllp\, which will analyze in further sections are empirical risk minimization (ERM) and stochastic gradient descent (SGD). For ERM, given a hypothesis $h$ and a loss function,  we let $g_h(x, y) = \ell(h(x), y)$ and the corresponding soft label corrected loss $\tL(h(x), \alpha) = \tg_h(x, \alpha)$. To provide regret guarantees using SGD over bags in a parameter space $\cW$ and loss function $\ell \colon \cW \times \cX \times \cY \to \Rset$, we use \easyllp\, to estimate the gradient of the loss function with respect to a parameter $\bw \in \cW$ by letting $g_\bw(x, y) = \nabla_{\bw} \ell(\bw, x, y)$ and its corresponding soft label corrected function $\tg_\bw(x, \alpha) = \nabla \tL(\bw, x, \alpha)$.

\section{ERM with Label Proportions}\label{s:erm}
Given a hypothesis space $\hyps$, let $\ell$ be a loss function as defined in Section \ref{sec:setup}. %Also assume that $\ell$ is bounded. 
Given a collection of bags
$\mathcal{S} = \{ (\bag_i, \alpha_i),\,i \in [n]\}$ of size $k$,
our learning algorithm simply finds $h \in \hyps$ that minimizes the empirical risk constructed via the soft label corrected loss:
\begin{equation}\label{e:ermsoftlabel}
     \sum_{i=1}^n \sum_{j=1}^k \tL(h(x_{ij}),\alpha_i).
\end{equation}
The main advantage of our algorithm lies in its simplicity and generality. Indeed, our algorithm can be used for any loss function and any hypothesis set. This is in stark contrast, e.g., to the work of \cite{QuadriantoSCL08}, whose framework is only applicable to the logistic loss and (generalized) linear models. From a practical standpoint, our approach can also leverage existing learning infrastructures, as the only thing that needs to be specified is a different loss function --- which in frameworks like Tensorflow, JAX and PyTorch requires only minimal coding. This differs from other approaches to assigning surrogate labels which may require solving combinatorial optimization problems like, e.g., \cite{dulac2019deep}.

The following theorem provides learning guarantees for minimizing the above empirical loss. Our guarantees are given in terms of the well-known Rademacher complexity of a class of functions.
\begin{definition}
Let $\mathcal{Z}$ be an arbitrary input space and let $\mathcal{G} \subset \{g \colon \mathcal{Z} \to \Rset\}$ be a collection of functions over $\mathcal{Z}$. Let $\cD$ be a distribution over $\mathcal{Z}$ and $\mathcal{S} = \{z_1, \ldots, z_m\}$ be an i.i.d. sample. The Rademacher complexity of $G$ is given by
\begin{equation*}
    \frac{1}{n}\rad_n(\mathcal{G}) = \E_{\mathcal{S}, \boldsymbol{\sigma}} \left[\sup_{g \in \mathcal{G}} \sum_{i=1}^n g(z_i) \sigma_i\right]
\end{equation*}
where $\boldsymbol{\sigma} = (\sigma_1, \ldots, \sigma_n) \in \{-1, 1\}^n$ is uniformly distributed. 
\end{definition}
\begin{restatable}{theorem}{thmUCB}
\label{thm:ucb}
Let $\delta > 0$, $\mathcal{S} = \{ (\bag_i, \alpha_i),\,i \in [n]\}$ be a collection of $n$ bags of size $k$.
%$\cS = (\bag_1, \alpha_1), \ldots, (\bag_n, \alpha_n)$ denote a sample of bags with their respective label proportions sampled from the joint distribution $\cD$.
Let $\sup_{\hy, y} \ell(\hy, y) \leq B$. Then the following bound holds uniformly for all $h \in \hyps$ with probability at least $1 - \delta$:
\begin{flalign*}
&\Bigl|\popl(h) - \frac{1}{nk}\sum_{i,j}\ell(h(x_{ij}), \alpha_i)\Bigl|\leq  &\nonumber \\ 
&C_n^k
\left(\rad_{kn}(\hyps_\ell^{(1)}) +
 \rad_{kn}(\hyps_\ell^{(0)})\right)  \nonumber 
 + \frac{4B}{n} + 4B \sqrt{\frac{k\log(2/\delta)}{2n}}~, \label{eq:ucb}
\end{flalign*}
where $C_n^k = 2\left(\sqrt{2 k \log (kn)} + 1 \right)$
, $\hyps_\ell^{(1)} = \{x \to \ell(h(x), 1) \colon h \in  \hyps\}$ and $\hyps_\ell^{(0)} = \{x \to \ell(h(x), 0) \colon h \in \hyps\}$. 
\end{restatable}
\begin{corollary}
With the notation of the previous theorem, let $\widehat{h}$ denote the minimizer of (\ref{e:ermsoftlabel}).
%$\frac{1}{n}\sum_{i=1}^n L(h, \bag_i, \alpha_i)$
Then with probability at least $1 - \delta$ over the sampling process we have:
\begin{equation*}
    \popl(\widehat{h}) \leq 
    \min_{h \in \hyps} \popl(h) + 2 \Gamma(k, n, \delta)~,
\end{equation*}
where $\Gamma(k, n, \delta) = C_n^k
\left(\rad_{kn}(H_\ell^{(1)}) +
 \rad_{kn}(H_\ell^{(0)})\right) 
   + \frac{4B}{n} + 4B \sqrt{\frac{k\log(2/\delta)}{2n}}$
\end{corollary}

\paragraph{Comparison to event level learning.} We can now compare the bound from Theorem~\ref{thm:ucb} to standard learning bounds for instance level learning like that of \cite{mohri}. Assuming we had access to a labeled i.i.d. sample $(x_{ij}, y_{ij})$ of size $kn$, Theorem 3.3 in \cite{mohri} ensures that with probability at least $1 - \delta$ the following bounds holds for all $h \in \hyps$:
\begin{equation}
    \label{eq:event_level}
  \biggl|\popl(h) - \frac{1}{nk}\sum_{i,j}\ell (h(x_{ij}, y_{ij}))\biggr| \leq 2 \rad_{nk}(\mathcal{H}_\ell)
+ B \sqrt{\frac{\log(2/\delta)}{2nk}}~,
\end{equation}
where $\hyps_\ell = \{(x, y) \to \ell(h(x), y) \colon h \in \hyps\}$. Note that under the weak assumption that the Rademacher complexities $\rad_{kn}(\hyps_\ell^{(r)})$, $r \in\{ 0, 1\}$, are of the same order as $\rad_{kn}(\hyps_\ell)$, the main difference between the bound in Theorem \ref{thm:ucb} and \eqref{eq:event_level} is simply an extra factor $C_n^k\in \tilde{O}(\sqrt{k})$ in the complexity term and a factor $\sqrt{k}$ multiplying the confidence term.
% last term of \eqref{eq:ucb}. 
That is, we achieve similar guarantees to event level learning by increasing the sample size by a factor of roughly $k$.

\section{SGD with Label Proportions}\label{s:sgd}
We now focus on understanding the effect of label proportions on another very popular learning algorithms, stochastic gradient descent (SGD).

Corollary \ref{cor:debiasedLoss} and Lemma \ref{l:derand} deliver unbiased estimates of the gradient which can be naturally plugged into any SGD algorithm~(e.g., \cite{Shamir013}), and one would hope for an upper bound on the excess risk if the learning task at hand leads to a convex optimization problem. The difficulty is that, even if each gradient in a given bag is individually unbiased, the gradients are correlated since they depend on the label proportion computed on the bag. A simple way around it is to pick a single item uniformly at random from the bag to update the model parameters. This is a slight departure from what we considered for ERM, but it both makes our SGD analysis easier and does not affect asymptotic performance.

This approach is presented in Algorithm \ref{alg:pickone}. The operator $\Pi_{\mathcal{W}}$ projects the parameters back to the convex domain $\mathcal{W}$, if the gradient update pushed them out of it.

\begin{algorithm}[ht!]
  \caption{SGD Using Pick-One from Each Bag \label{alg:pickone}}
  \begin{algorithmic}[1]
    \STATE {\bf Input }{$\mathcal{S} = \{ (\bag_i, \alpha_i),\,i \in [n]\}$,  $\{\eta_{i} : i\in [n]\}$, $\bw_0$}
    \STATE $\bw \leftarrow \bw_0$
    \FOR{$t \in\{1,\ldots , n\}$}
    \STATE Pick $j$ from $[k]$ uniformly at random
    \STATE Update the model parameter $\bw_t$ as
    $$\bw_{t+1} \leftarrow \Pi_{\mathcal{W}} \big( \bw_{t} -\eta_i \tg(x_{tj}, \ty_t) \big)$$
    with $\tg(x_{tj}, \ty_t)$ defined in (\ref{eq:tlsimple})
    %$\ty_t \sim \text{Bernoulli}(\alpha_t)$ 
    \hfill {\small{\texttt{//surrogate label}}}
    \STATE {\bf OR} Update the model parameter $\bw_t$ as
    $$\bw_{t+1} \leftarrow \Pi_{\mathcal{W}} \big( \bw_{t} -\eta_i \tg(x_{tj}, \alpha_t) \big)$$
    with $\tg(x_{tj}, \alpha_t)$ defined in (\ref{eq:simple_tlalpha}) \hfill {\small{\texttt{//soft label}}}
    \ENDFOR
    \STATE Update the model parameter with 
    
    \STATE{\textbf{Return} $\bw_{n+1}$}
  \end{algorithmic}
\end{algorithm}

We give two versions of the update rule, one based on a single surrogate label as in (\ref{eq:tlsimple}), the other based on a single (derandomized) soft label as in (\ref{eq:simple_tlalpha}). The error depends on the squared norm of these gradients, which is of order $O(k^2)$ using surrogate labels and $O(k)$ using soft labels, as presented in the next theorem.
\begin{restatable}{theorem}{thmConvexPickone}\label{thm:convex_pickone}
Suppose that $F(\bw) = \E[ \ell_{\bw} (x, y) ]$ is convex,  $\E[\| g_{\bw_t} (x, y) \|^2 ] \le G^2$ for all $t\in [n]$, 
and $\sup_{\bw, \bw'\in \mathcal{W}} \|\bw -\bw' \|\le D$. Consider Algorithm \ref{alg:pickone} using surrogate labels, run with $\eta_t = 1/(k \cdot \sqrt{t})$.
%that is run on $\{\bag_1, \ldots, \bag_n\}$ where $\bag_i = \{(x_{ij},y_{ij}) \colon j \in [k]\}$. 
Then for any $n>1$, we have
\[
\E[ F (\bw_n ) - F(\bw^*)] \le k\cdot \left(D^2 + 5G^2 \right) \frac{2+\log(n)}{\sqrt{n}} \,.
\]
Consider now Algorithm \ref{alg:pickone} using soft labels, run with $\eta_t = 1/ \sqrt{k \cdot t}$.
Then for any $n>1$ we have
\[
\E[ F (\bw_n ) - F(\bw^*)] \le \sqrt{k}\cdot \left(D^2 + 5G^2 \right) \frac{2+\log(n)}{\sqrt{n}} \,.
\]
\end{restatable}

Note that the error is higher by a multiplicative factor of $\sqrt{k}$ when using surrogate labels in the gradient update compared to soft labels.  Recall that the error of SGD with $kn$ indivually labeled examples decreases like $O\left(\frac{1}{\sqrt{kn}}\right)$. Compared with the above bound, we see that, similar to the ERM scenario,  the regret bound increases by a factor of $k$.

\section{Generalization to Multi-Class}\label{s:multi}

Thus far we have restricted our analysis to binary classification problems. This was mostly done for ease of exposition. We now show that \easyllp can easily generalize to the multi-class scenario. This is in contrast to other methods like that of \cite{dulac2019deep} whose generalization to multi-class requires solving an optimization problem for every gradient step. For the rest of this section we let $C$ denote the number of classes and $\cY = \{1, \ldots, C\}$ denote the label space. The result below is proven in the appendix and is a generalization of its binary counterpart.

\begin{restatable}{theorem}{thmMulticlass}
\label{thm:multiclass}
Let $k> 0$ and  $(x_1, y_1),\ldots, (x_,y_k)$ be a sample from distribution $\cD$. For each class $c$, let $\alpha_c = \frac{1}{k}\sum_{i=1}^k \ind_{y_i = c}$ denote the fraction samples with label $c$. Let $\bag = \{x_1, \ldots, x_k\}$ and define $\boldsymbol{\alpha} = (\alpha_1, \ldots, \alpha_k)$. Finally, let $p_c = \PP(y = c)$.  Let $g\colon \cX \times \cY \to \Rset^d$ be any (measurable) function over features and labels. Define the soft label corrected function $\tg$ as:
\begin{equation*}
    \tg(x, \boldsymbol{\alpha}) = \sum_{c=1}^C \big(k\alpha_c - (k-1)p_c\big) g(x, c).
\end{equation*}
For any $x_i \in \bag$, the soft label corrected function satisfies:
\begin{equation*}
    \E[\tg(x_i, \boldsymbol{\alpha})] = \E_{x, y \sim \cD}[g(x, y)]~.
\end{equation*}
\end{restatable}

\section{Experiments}\label{sec:experiments}
We now present an empirical evaluation of our algorithm, comparing its performance against other approaches to LLP.

\paragraph{Baselines.} The baselines we compare to are algorithms for learning event level classifiers. We do not compare against algorithms that try to match the bag label proportions as they are aimed at different tasks. In fact, the empirical evaluation achieved by those algorithms requires very specific data generation processes, making those algorithms unlikely to work on i.i.d. data (see, e.g., the discussion in \cite{ScottZ20}).
\begin{itemize}
    \item {\bf Event level. } As an obvious baseline, the learner is trained with access to event level labeled examples $(x_{ij}, y_{ij})$. This baseline delivers an upper bound on the performance of any label proportion algorithm. 
    \item {\bf Label generation. } This algorithm is introduced in  \cite{dulac2019deep}. Their algorithm generates artificial labels at every training round to match the observed label proportions. Labels are generated to match the current predictions as much as possible. 
    \item{\bf MeanMap. } The  algorithm of \citet{QuadriantoSCL08}. The algorithm is specialized for logistic regression with linear models and uses linear algebra manipulations on the label proportions to obtain an estimate of the mean operator $\frac{1}{nk}\sum_{i,j} x_{ij}y_{ij}$. We only use this algorithm when learning linear classifiers as it does not generalize to arbitrary hypothesis sets.
\end{itemize}

\paragraph{Datasets.} We tested on the following binary classification datasets:
\begin{itemize}
    \item \textit{MNIST.} We use the MNIST dataset with 11 different binarizations of the original multi-class labels: even-vs-odd, and a one-vs-all binarization for each of the 10 digits.
    \item \textit{Cifar10.} We use the Cifar10 dataset with 11 different binarizations of the original multi-class labels: animal-vs-machine, and a one-vs-all binarization for each of the 10 classes.
\end{itemize}

\paragraph{Models.} We evaluate the methods under two different model classes: a linear model for all tasks, a deep convolutional network for MNIST and CIFAR 10 and a deep neural network for Kaggle CTR dataset. The specifications for the networks will be made available with our code at publication time. All models are trained using the binary cross-entropy loss.

\paragraph{Methodology}
For each LLP method, model class, and dataset, we use the LLP method to train a model using bags of size $k = 2^r$ for $r = 1,2, \ldots, 10$. Bags are generated by randomly grouping examples together. Each algorithm is trained using the \textsc{ADAM} optimizer for 40 rounds. We replicate each learning task 30 times. Unlike the experimental setup of previous work \cite{ScottZ20}, for experiment replica our bags are generated only once and label proportions are fixed before training starts. By fixing the bags throughout the whole learning period, we observe a slight regression on the test accuracy for all algorithms as we increase the number of epochs. For this reason we report, for each algorithm, the best test accuracy across all training rounds. It is worth highlighting that this regression effect has already been observed in past investigations, e.g., \cite{ScottZ20}, and is likely the effect of overfitting to the soft labels. As we discuss later on, we expect this overfitting effect to be diminished by larger datasets and ensuring each example is used only once in the training process.

In Figure~\ref{fig:bagsize_vs_perf} we plot
test accuracy of the model vs. bag size.
Each curve is averaged over the 30 independent runs, each with a different random partition of the data into bags of size $k$ and random model weight initialization.

\begin{figure}[ht!]
    \includegraphics[width=0.22\textwidth]{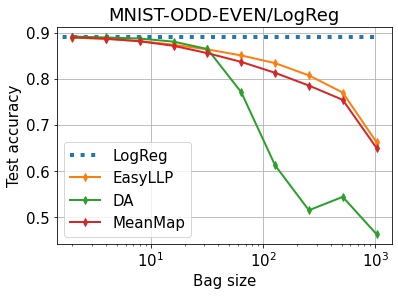}\quad
    \includegraphics[width=0.22\textwidth]{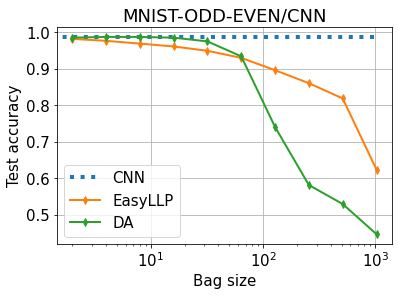}\\
    \includegraphics[width=0.22\textwidth]{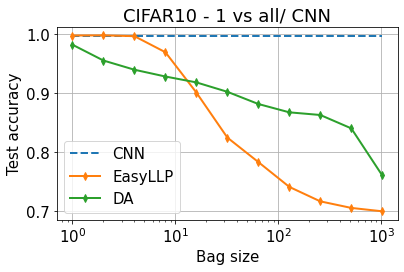}\quad
    \includegraphics[width=0.22\textwidth]{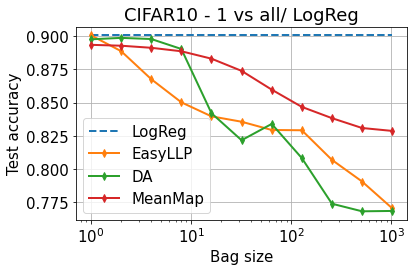}\\
    \includegraphics[width=0.22\textwidth]{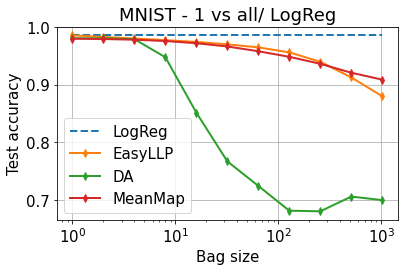}\quad
    \includegraphics[width=0.22\textwidth]{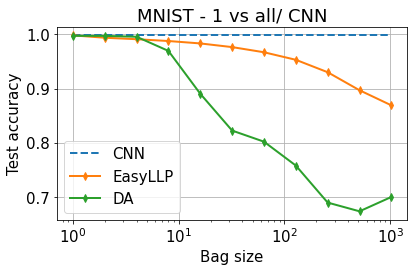}
    \caption{Performance of various methods trained with label proportion. Here, the event level performance is given by either a logistic regression algorithm (``LogReg") or a CNN. ``DA" denotes the method from \cite{dulac2019deep}. }
    \label{fig:bagsize_vs_perf}    \end{figure}

Finally, since there are 10 different one-vs-all tasks for MNIST and Cifar10, we report the average test accuracy over all 10 tasks for various bag sizes.
That is, each point in the one-vs-all plot corresponds to the average of 300 independent runs of LLP, with 30 of those runs belonging to each of the 10 one-vs-all tasks.

Some of the takeaways we can get from the results is that for linear models our algorithm performs very similar to that of MeanMap. We believe there may be a deeper connection between our algorithm and MeanMap, i.e., our algorithm is a strict generalization of MeanMap, but this connection is left as an open research question. We also see that the label generation algorithm consistently underperforms MeanMap and our algorithm. The main advantage of our algorithm however can be observed when we use neural networks as our base model class. The flexibility of choosing an arbitrary model class is what makes our proposed algorithm state of the art. On the other hand, the comparison to DA is a bit mixed (other plots are in the Appendix \ref{app:cifar10}), and we believe this may deserve further investigation.

\paragraph{Comparison of unbiased loss and training loss.}

We now demonstrate that our unbiased loss is indeed (not only theoretically but in practice) a good estimator for the empirical instance level loss.
We train a CNN model on the MNIST even-vs-odd task with \easyllp{} for 200 epochs with bags of size 32 and record the empirical instance level training loss and the estimated training loss from bags.
\Cref{fig:tracking} compares the trajectories of the loss and loss estimates as the model is trained for more epochs.
At the beginning of each epoch, we shuffle the training data and create consecutive bags of examples.
The curves are averaged over 30 independent runs and the shaded bands show one standard deviation.
We see that the loss estimates closely track the training loss, albeit with larger variance.

\begin{figure}[h!]
    \centering
    \includegraphics[width=0.55\columnwidth]{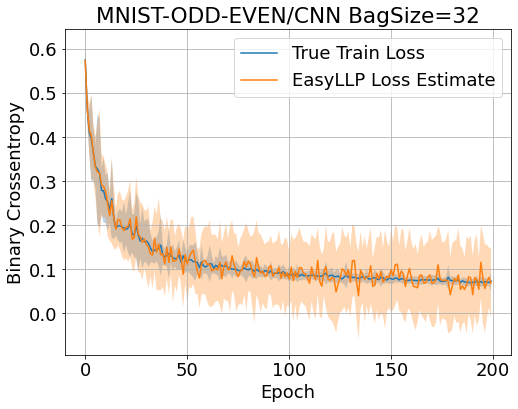}
    \caption{Depicts the true training loss (measured using instance-level access to the training data) compared to the loss estimates computed from bags.
    We average over 30 runs and the error bands show $\pm$ one standard deviation.}
    \label{fig:tracking}
\end{figure}

\section{Conclusions}\label{s:concl}
We have introduced \easyllp, a novel and flexible approach to LLP for  classification. The method is widely applicable and its theoretical underpinning  assumes the bags are non-overlapping i.i.d. samples drawns from the same (unknown) distribution.
For the sake of illustration, we started off by considering the generation of random surrogate labels and an associated debiased loss, and then argued that a simple derandomization helps decrease the variance of the resulting estimators. We illustrated applications to ERM and SGD, and proved that LLP performance against event level loss is only a factor $\sqrt{k}$ worse than full supervision. Finally, we conducted an experimental investigation against representative baselines in binary classification, and showed that, on the tested datasets, our method is either on par with or superior to its competitors.

As future activity, we would like to have a better understanding of the connection between our method and other methods proposed  in the literature. For instance, when applied to linear models we suspect that \easyllp gets tightly related to MeanMap. Also, a more thorough experimental comparison is ongoing that encompasses more baselines and more benchmark datasets.

%\newpage

%\bibliography{paper_arxiv}

\begin{thebibliography}{25}
\providecommand{\natexlab}[1]{#1}
\providecommand{\url}[1]{\texttt{#1}}
\expandafter\ifx\csname urlstyle\endcsname\relax
  \providecommand{\doi}[1]{doi: #1}\else
  \providecommand{\doi}{doi: \begingroup \urlstyle{rm}\Url}\fi

\bibitem[pri(2022)]{privacy_sandbox}
Private aggregation {API}, 2022.
\newblock URL
  \url{https://developer.chrome.com/docs/privacy-sandbox/private-aggregation/}.

\bibitem[Bortsova et~al.(2018)Bortsova, Dubost, Orting, Katramados, Hogeweg,
  Thomsen, Wille, and de~Bruijne]{bor+18}
Bortsova, G., Dubost, F., Orting, S., Katramados, I., Hogeweg, L., Thomsen, L.,
  Wille, M., and de~Bruijne, M.
\newblock Deep learning from label proportions for emphysema quantification.
\newblock In \emph{Medical Image Computing and Computer Assisted Intervention
  -- MICCAI 2018}, pp.\  768--776, 2018.

\bibitem[de~Freitas \& Kuck(2005)de~Freitas and Kuck]{dfk05}
de~Freitas, N. and Kuck, H.
\newblock Learning about individuals from group statistics.
\newblock In \emph{Proceedings of the 21st Conference in Uncertainty in
  Artificial Intelligence (UAI 2005)}, pp.\  332--339, 2005.

\bibitem[Dery et~al.(2018)Dery, Nachman, Rubbo, and Schwartzman]{dery+18}
Dery, L., Nachman, B., Rubbo, F., and Schwartzman, A.
\newblock Weakly supervised classification for high energy physics.
\newblock \emph{Journal of Physics: Conference Series, 1085:042006}, 2018.

\bibitem[Ding et~al.(2017)Ding, Li, , and Yu]{ding+17}
Ding, Y., Li, Y., , and Yu, W.
\newblock Learning from label proportions for sar image classification.
\newblock \emph{EURASIP Journal on Advances in Signal Processing}, 2017.

\bibitem[Dulac-Arnold et~al.(2019)Dulac-Arnold, Zeghidour, Cuturi, Beyer, and
  Vert]{dulac2019deep}
Dulac-Arnold, G., Zeghidour, N., Cuturi, M., Beyer, L., and Vert, J.-P.
\newblock Deep multi-class learning from label proportions.
\newblock \emph{arXiv preprint arXiv:1905.12909}, 2019.

\bibitem[Ledoux \& Talagrand(1991)Ledoux and Talagrand]{talagrand}
Ledoux, M. and Talagrand, M.
\newblock \emph{Probability in Banach Spaces: isoperimetry and processes},
  volume~23.
\newblock Springer Science \& Business Media, 1991.

\bibitem[Li \& Taylor(2015)Li and Taylor]{lt15}
Li, F. and Taylor, G.
\newblock Alter-cnn: An approach to learning from label proportions with
  application to ice-water classification.
\newblock In \emph{NIPS workshop on Learning and privacy with incomplete data
  and weak supervision}, 2015.

\bibitem[Lu et~al.(2021)Lu, Lei, Niu, Sato, and Sugiyama]{lu+21}
Lu, N., Lei, S., Niu, G., Sato, I., and Sugiyama, M.
\newblock Binary classification from multiple unlabeled datasets via surrogate
  set classification.
\newblock In \emph{Proceedings of the 38th International Conference on Machine
  Learning}, pp.\  7134--7144, 2021.

\bibitem[Lu \& Sugiyama(2019)Lu and Sugiyama]{lu+19}
Lu, N., N. G. M. A.~K. and Sugiyama, M.
\newblock On the minimal supervision for training any binary classifier from
  only unlabeled data.
\newblock In \emph{Proc. ICLR, 2019}, 2019.

\bibitem[Mohri et~al.(2018)Mohri, Rostamizadeh, and Talwalkar]{mohri}
Mohri, M., Rostamizadeh, A., and Talwalkar, A.
\newblock \emph{Foundations of machine learning}.
\newblock MIT press, 2018.

\bibitem[Musicant et~al.(2007)Musicant, Christensen, and Olson]{mco07}
Musicant, D., Christensen, J., and Olson, J.
\newblock Supervised learning by training on aggregate outputs.
\newblock In \emph{Proceedings of the 7th IEEE International Conference on Data
  Mining (ICDM 2007)}, 2007.

\bibitem[Patrini et~al.(2014)Patrini, Nock, Rivera, and Caetano]{pat+14}
Patrini, G., Nock, R., Rivera, P., and Caetano, T.
\newblock ({A}lmost) no label no cry.
\newblock In \emph{Proc. NIPS}, 2014.

\bibitem[Qi et~al.(2017)Qi, Wang, Meng, , and Niu]{qi+17}
Qi, Z., Wang, B., Meng, F., , and Niu, L.
\newblock Learning with label proportions via npsvm.
\newblock \emph{IEEE Transactions on Cybernetics}, 47(10):\penalty0 3293--3305,
  2017.

\bibitem[Quadrianto et~al.(2008)Quadrianto, Smola, Caetano, and
  Le]{QuadriantoSCL08}
Quadrianto, N., Smola, A.~J., Caetano, T.~S., and Le, Q.~V.
\newblock Estimating labels from label proportions.
\newblock In \emph{ICML}, volume 307 of \emph{ACM International Conference
  Proceeding Series}, pp.\  776--783. ACM, 2008.

\bibitem[Rueping(2010)]{rue10}
Rueping, S.
\newblock {SVM} classifier estimation from group probabilities.
\newblock In \emph{ICML}, Proceedings of the 27th International Conference on
  Machine Learning, pp.\  911--918, 2010.

\bibitem[Saket(2021)]{NEURIPS2021_33b3214d}
Saket, R.
\newblock Learnability of linear thresholds from label proportions.
\newblock In \emph{Advances in Neural Information Processing Systems},
  volume~34, pp.\  6555--6566. Curran Associates, Inc., 2021.

\bibitem[Saket(2022)]{52071}
Saket, R.
\newblock Algorithms and hardness for learning linear thresholds from label
  proportions.
\newblock In \emph{Proc. NeurIPS'22}, 2022.

\bibitem[Scott \& Zhang(2020)Scott and Zhang]{ScottZ20}
Scott, C. and Zhang, J.
\newblock Learning from label proportions: A mutual contamination framework.
\newblock In \emph{Advances in Neural Information Processing Systems 33: Annual
  Conference on Neural Information Processing Systems 2020, NeurIPS 2020},
  2020.

\bibitem[Shamir \& Zhang(2013)Shamir and Zhang]{Shamir013}
Shamir, O. and Zhang, T.
\newblock Stochastic gradient descent for non-smooth optimization: Convergence
  results and optimal averaging schemes.
\newblock In \emph{Proceedings of the 30th International Conference on Machine
  Learning, {ICML} 2013}, volume~28 of \emph{{JMLR} Workshop and Conference
  Proceedings}, pp.\  71--79. JMLR.org, 2013.

\bibitem[Shi et~al.(2019)Shi, Cui, Chen, and Qi]{shi+19}
Shi, Y., Cui, L., Chen, Z., and Qi, Z.
\newblock Learning from label proportions with pinball loss.
\newblock \emph{International Journal of Machine Learning and Cybernetics},
  10(1):\penalty0 187--205, 2019.

\bibitem[Stolpe \& Morik(2011)Stolpe and Morik]{10.1007/978-3-642-23808-6_23}
Stolpe, M. and Morik, K.
\newblock Learning from label proportions by optimizing cluster model
  selection.
\newblock In \emph{Machine Learning and Knowledge Discovery in Databases}, pp.\
   349--364. Springer Berlin Heidelberg, 2011.

\bibitem[Sun et~al.(2017)Sun, Sheldon, and O’Connor]{sun+17}
Sun, T., Sheldon, D., and O’Connor, B.
\newblock A probabilistic approach for learning with label proportions applied
  to the us presidential election.
\newblock In \emph{IEEE International Conference on Data Mining (ICDM)}, pp.\
  445--454, 2017.

\bibitem[Yu et~al.(2013)Yu, Liu, Kumar, Jebara, and Chang]{yu+13}
Yu, F., Liu, D., Kumar, S., Jebara, T., and Chang, S.
\newblock $\alpha$-svm for learning with label proportions.
\newblock In \emph{ICML}, Proceedings of the 30th International Conference on
  Machine Learning, pp.\  504--512, 2013.

\bibitem[Zhang et~al.(2022)Zhang, Wang, and Scott]{z+22}
Zhang, J., Wang, Y., and Scott, C.
\newblock Learning from label proportions by learning with label noise.
\newblock In \emph{Proc. NeurIPS'22}, 2022.

\end{thebibliography}
%\bibliographystyle{icml2023}

%%%%%%%%%%%%%%%%%%%%%%%%%%%%%%%%%%%%%%%%%%%%%%%%%%%%%%%%%%%%%%%%%%%%%%%%%%%%%%%
%%%%%%%%%%%%%%%%%%%%%%%%%%%%%%%%%%%%%%%%%%%%%%%%%%%%%%%%%%%%%%%%%%%%%%%%%%%%%%%
% APPENDIX
%%%%%%%%%%%%%%%%%%%%%%%%%%%%%%%%%%%%%%%%%%%%%%%%%%%%%%%%%%%%%%%%%%%%%%%%%%%%%%%
%%%%%%%%%%%%%%%%%%%%%%%%%%%%%%%%%%%%%%%%%%%%%%%%%%%%%%%%%%%%%%%%%%%%%%%%%%%%%%%

\newpage
\appendix
\onecolumn

\section{Proofs}
Throughout this appendix, we denote by $\ind_{A}$ the indicator function of the predicate $A$ at argument. %Moreover, $\mathcal{D}_{\cX}$ denotes the marginal distribution of $\cD$ on $\cX$. 

\subsection{Proof of Proposition \ref{prop:bias}}
\propBias*
\begin{proof}
Let $(x_1,y_1),\ldots,(x_k,y_k)$ be a sample drawn i.i.d. from $\cD$ and let $\alpha = \frac{1}{k} \sum_{i=1}^k y_i$ be the label proportion. 
Fix $j$ and let $\ty_j \sim \bernoulli(\alpha)$ denote the surrogate label for example $x_j$.
The distribution of $\ty_j$ is equivalent to first sampling an index $I$ uniformly at random from $[k]$ and then setting $\ty_j = y_I$.
With this, we have
\begin{align*}
    \E[g(x_j, \ty_j)]
    = \sum_{i=1}^k \E[g(x_j, \ty_j) \mid I = i] \cdot \Pr(I = i)
    = \frac{1}{k}\sum_{i=1}^k \E[g(x_j, y_i)].
\end{align*}
When $i \neq j$, we have that $x_i$ and $y_j$ are independent, which gives
\[
\E[g(x_j, y_i)] = (1-p)\E[g(x_j, 0)] + p \E[g(x_j, 1)],
\]
which no longer depends on the index $i$.
Therefore, we have
\[
\E[g(x_j, \ty_j)]
= \frac{1}{k} \E[g(x_j, y_j)]
+ \frac{(k-1)(1-p)}{k} \E[g(x_j, 0)]
+ \frac{(k-1)p}{k} \E[g(x_j, 1)],
\]
as required.
\end{proof}

\subsection{Proof of Theorem \ref{t:gradient_squared}}

In this section we prove \Cref{t:gradient_squared}, which we recall now:
\thmGradientSquared*

For simplicity of notation let us define the following quantities:
\begin{equation*}
    \tg_i = \tg(x_i, \alpha) \qquad g_{i0} = g(x_i,0), \qquad g_{i1} = g(x_i,1)~.
\end{equation*}

Let also $A = k(\alpha - p) + p$. With this notation we have that 
\begin{equation*}
    \tg_i = A g_{i1} + (1 - A) g_{i0}. 
\end{equation*}
The proof of Theorem \ref{t:gradient_squared} will be a consequence of the following lemmas.

\begin{lemma}
\label{lemma:bound_on_single}
Using the above notation, the following inequality holds
\begin{equation*}
    \E\left[\left\|\frac{1}{k} \sum_{i=1}^k\tg_i\right\|^2\right] \leq \E[\|\tg_1\|^2]~.
\end{equation*}
\end{lemma}
\begin{proof}
By simple linear algebra and the fact that the $\tilde g_i$ are equally distributed we have
\begin{align}
\E\left[\Bigl|\Bigl|\frac{1}{k} \sum_{i=1}^k\tilde g_i\Bigl|\Bigl|^2\right] = \frac{1}{k} \E[||\tilde g_1||^2] + \frac{k-1}{k}\,\E[\langle\tilde g_1,\tilde g_2\rangle]~. \label{eq:norm_decomp}    
\end{align}
Moreover,
\[
0 \leq \E[||\tilde g_1- \tilde g_2||^2] = 2\E[||\tilde g_1||^2] - 2\E[\langle\tilde g_1, \tilde g_2\rangle]
\]
implies
\[
\E[\langle\tilde g_1, \tilde g_2\rangle] \leq \E[||\tilde g_1||^2]~.
\]
Replacing this inequality in \eqref{eq:norm_decomp} yields
\[
\E\left[\Bigl|\Bigl|\frac{1}{k} \sum_{i=1}^k\tilde g_i\Bigl|\Bigl|^2\right] \leq \E[||\tilde g_1||^2]~,
\]
which is the claimed result.
\end{proof}
\begin{lemma}
\label{lemma:single_var}
Let $M = \sup_{x, y} \|g(x,y)\|$. For any index $i$ the following inequality holds:
\begin{equation}
    \E[\|\tg_i\|^2]  \leq  9M^2 + (k-1)p(1-p)\E_{x\sim D}[\|g(x, 1) - g(x,0)\|^2] 
\end{equation}
\end{lemma}
\begin{proof}
Fix an index $i$ and rewrite $A$ as 
\begin{align*}
    A &= \sum_{j=1}^k y_j - (k-1)p
    = y_i + \sum_{j\neq i} y_j - (k-1)p
    = y_i + B
\end{align*}
where $B$ is a centered binomial random variable of parameters $(k-1, p)$ independent of $x_i,y_i$. This entails, 
\begin{equation}
    \E[B] = 0  \quad \text{and} \quad
    \E[B^2 ] = (k-1) p(1-p)~. \label{eq:binom_expec1}
\end{equation}
Setting for brevity $g_{i1} = g(x_i,1)$, $g_{i0} = g(x_i,0)$,
we thus have
\begin{align*}
    \E[\|\tg_i\|^2]  &= 
    \E[\|(B + y_i)g_{i1} + (1 - B - y_i)g_{i0}\|^2] \\
    &= \E[\|y_i(g_{i1} - g_{i0}) + g_{i0} - B(g_{i1} - g_{i0})\|^2]\\
    &= \E[\|Z_i + B(g_{i1} - g_{i0})\|^2],
\end{align*}
where $Z_i \in \Rset^d$ is given by $Z_i = y_i(g_{i1} - g_{i0}) + g_{i0}$. Expanding the above expression we see that 
\begin{align*}
    \E[\|Z_i + B(g_{i1} - g_{i0})\|^2]
    &= \E[\|Z_i\|^2] + 2 \E[B\langle Z_i,(g_{i1} - g_{i0}\rangle] +  \E[B^2 \|g_{i1} - g_{i0}\|^2] \\
    &=\E[\|Z_i\|^2] + 2 \E[B]\E[\langle Z_i,(g_{i1} - g_{i0}\rangle] +  \E[B^2]\E[\|g_{i1} - g_{i0}\|^2],
\end{align*}
where we have used the fact that $Z_i, g_{i0}, g_{i1}$ are all functions of $x_i, y_i$ and $B$ is independent of of these variables. Using \eqref{eq:binom_expec1} and the fact that $\|Z_i\| \leq 3 M$ we obtain the result.
\end{proof}
\begin{lemma}
\label{lemma:cross_prod}
Let $M = \sup_{x,y}\|g(x,y)\|$. For any pair of indices $i,j$ the following inequality holds
\begin{equation*}
    \E[\inner{\tg_i, \tg_j}] \leq (k-2)p(1-p) \|\E[g_{i1} - g_{i0}]\|^2 + 36M
\end{equation*}
\end{lemma}
\begin{proof}
Fix $i, j$. As in the previous lemma, let us rewrite $A$ as
\begin{align}
    A &= y_i + y_j  - p + \sum_{r \neq i,j} y_r - (k-2)p\\
    &= R + B
\end{align}
where $R = y_i + y_j - p$ and $B$ is a centered binomial random variable of parameters $(k-2, p)$. Moreover $B$ is independent of $(x_i, x_j, y_i,y_j)$. 
We proceed to calculate the desired expectation
\begin{align*}
    \E[\inner{\tg_i, \tg_j}]
    &= \E[\inner{A (g_{i1} - g_{i0}) + g_{i0},
    A(g_{j1} - g_{j0}) + g_{j0}}]\\
    &= \E[\inner{B (g_{i1} - g_{i0}) + R(g_{i1} - g_{i0}) +  g_{i0},
    B(g_{j1} - g_{j0}) + R((g_{j1} - g_{j0}) + g_{j0}}]~,
\end{align*}
where, again, $g_{i1} = g(x_i,1)$, and $g_{i0} = g(x_i,0)$.
Using a similar argument as in the previous lemma, it is not hard to see that the above expression simplifies to:
\begin{equation*}
    \E[\inner{\tg_i, \tg_j}] = 
    (k-2)p(1-p) \E[\inner{g_{i1} - g_{i0}, g_{j1}- g_{j0}}] + \E[\inner{R((g_{i1} - g_{i0}) + g_{i0},
    R((g_{j1} - g_{j0}) + g_{j0}}]~.
\end{equation*}
Using the fact that $g_{i1} - g_{i0}$ and $g_{j1} - g_{g0}$ are independent and identically distributed as well as Cauchy-Schwartz inequality for the second term we obtain the following bound 
$$
\E[\inner{\tg_i, \tg_j}] \leq  
    (k-2)p(1-p) \|\E[g_{i1} - g_{i0}]\|^2 + \E[9MR^2]
    \leq (k-2)p(1-p) \|\E[g_{i1} - g_{i0}]\|^2 + 36M~,
    $$
as claimed.
\end{proof}
\begin{lemma}
Under the same notation as the previous lemma we have
\begin{align*}
    \E\left[\left\|\frac{1}{k}\sum_{i=1}^k \tg_i\right\|^2\right]& \leq 
    \frac{1}{k}\left(3M + k p(1-p)\E[\|g(x,1) - g(x,0)\|]^2\right) \\
 &\qquad + \frac{(k-1)}{k}\left(36M^2 + (k-2)p(1-p)\|\E[g(x,1) - g(x,0)]\|^2\right)
\end{align*}
\end{lemma}
\begin{proof}
Using \eqref{eq:norm_decomp}
%we and the fact that samples are i.i.d 
we can write
\begin{align*}
        \E\left[\left\|\frac{1}{k}\sum_{i=1}^k \tg_i\right\|^2\right] =
       \frac{1}{k} \E[||\tilde g_1||^2] + \frac{k-1}{k}\,\E[\langle\tilde g_1,\tilde g_2\rangle]~.
\end{align*}
By applying Lemma~\ref{lemma:single_var} and Lemma~\ref{lemma:cross_prod} we can bound the above expression as
\begin{align*}
 \E\left[\left\|\frac{1}{k}\sum_{i=1}^k \tg_i\right\|^2\right]
 &\leq \frac{1}{k}\left(3M + k p(1-p)\E[\|g(x,1) - g(x,0)\|]^2\right) \\
 &\qquad+ \frac{(k-1)}{k}\left(36M^2 + (k-2)p(1-p)\|\E[g(x,1) - g(x,0)]\|^2\right)~,
\end{align*}
as anticipated.
\end{proof}

\subsection{Proof of Theorem~\ref{thm:ucb}}
The proof of the theorem depends on the following proposition.
\begin{proposition}
\label{prop:truncated}
For any $\eta > 0$ define $c(k,\eta, p) = \sqrt{ \frac{\log(1/\eta)}{2k}}$ let
\begin{equation*}
    \tL_\eta(\hy, \alpha) = \ell(\hy, \alpha)\ind_{|\alpha - p| \leq c(k, \eta)}.
\end{equation*}
be a truncation of $\tL$. Let $(\bag, \alpha)$ be a labeled bag.
%with its label proportion. 
Let $x \in \bag$. With probability at least $1 - \eta$ over the choice of $(\bag, \alpha$), we have $|\tL_\eta(h(x), \alpha) = \tL(h(x), \alpha)|$ for all $h \in \hyps$.
\end{proposition}
\begin{proof}
For any $h$, note that both losses agree unless $(\alpha - k) \geq c(k,\eta)$ but by Hoeffding's inequality this occurs with probability at most $\eta$.
\end{proof}

We now proceed to prove Theorem~\ref{thm:ucb}.
\thmUCB*
\begin{proof}
For $h \in \hyps$, bag $\bag$ and label proportion $\alpha$ define
\begin{equation*}
    \label{eq:bag_loss}
    L(h, \bag, \alpha) = \frac{1}{k}\sum_{x\in \bag} \tL(h(x), \alpha).
\end{equation*}

Let $\Phi(S) = \sup_{h \in \hyps} \E[\ell(h(x), y]) - \frac{1}{n}\sum_{i=1}^n L(h, \bag_i, \alpha_i)$~.

Let $(x_{ij}, y_{ij})_{i\in[n], j\in [k]}$ denote the \emph{instance} level sample. Let $\mathcal{S'} = (\bag_i', \alpha_i')_{i \in [k]}$ denote the sample obtained by switching a single sample $(x_{ij}, y_{ij})$ to $(x'_{ij}, y'_{ij})$. Without loss of generality assume we switch sample $(x_{11}, y_{11})$. Then by the subadditive property of the supremum and the fact that $(\bag_i, \alpha_i) = (\bag_i', \alpha_i')$ for $i \neq 1$ we have
\begin{align*}
    \Phi(\cS) - \Phi(\cS') \leq \frac{1}{n} \sup_{h \in \hyps} L(h,\bag_1, \alpha_1) - L(h, \bag_1', \alpha_1').
\end{align*}
If we expand the difference inside the supremum for a fixed $h$ we have:
\begin{align*}
    |L(h, \bag_1, \alpha_1) - L(h, \bag_1', \alpha_1')|
    \leq  \frac{1}{k}\sum_{j=1}^k |\tL(h(x_{1j}, \alpha_1) - \tL(h(x_{1j}'), \alpha'_1)|~.
\end{align*}

Now using the fact that $|\alpha_1- \alpha_1'|\leq \frac{1}{k}$ --- as only one label changed --- and $\tL$ is $2kB$-Lipchitz as a function of $\alpha$ we must have, for $j \neq 1$,
$$
|\tL(h(x_{1j}), \alpha_1) - \tL(h(x_{1j}'), \alpha'_1)|
= |\tL(h(x_{1j}), \alpha_1) - \tL(h(x_{1j}), \alpha'_1)|
\leq 2B~.
$$
On the other hand, using the fact that $k|\alpha - p| + \max(p, 1-p) \leq k+1$, and again the fact that $\tL$ is Lipchitz with respect to $\alpha$ we see that:

$$|\tL(h(x_{11}), \alpha_1) - \tL(h(x_{1j}'), \alpha'_1)| \leq (k+1)B + 2B = B(k+3).
$$
We thus have that 
$$
|L(h, B_1, \alpha_1) - L(h, B_1', \alpha_1')| \leq
\frac{1}{k}\left( B(k+3) + 2(k-1)B \right)
= \frac{(3 k + 1)}{k}B \leq 4B
$$
We therefore have 
$$\Phi(\cS) - \Phi(\cS') \leq \frac{4B}{n}$$.

By McDiarmid's inequality and the fact that we have $k n$ individual samples we thus have that with probability at least $1-\delta$: 
$$\Phi(S) \leq \E[\Phi(S)] + 4B\sqrt{\frac{k \log(1/\delta)}{2 n}}$$
We proceed to bound the expectation of $\Phi(S)$:
\begin{align*}
\E[\phi(\cS)] &= \E_{\cS}\left[\sup_{h\in \hyps} \E[\ell(h(x), y)] -\frac{1}{n}\sum_{i=1}^n L(h, \bag_i, \alpha_i)\right]     \\
&= \E_\cS\left[\sup_{h \in \hyps} \E_{\cS'}\left[\frac{1}{n} \sum_{i=1}^n L(h, \bag_i', \alpha_i')\right]- L(h, \bag_i,\alpha_i)\right] \\
&\leq \frac{1}{n}\E_{\cS, \cS'} \left[\sup_{h \in \hyps} 
\sum_{i=1}^n L(h, \bag_i', \alpha_i')- L(h, \bag_i,\alpha_i)\right]~,
\end{align*}
where $\cS'$ is another i.i.d. sample drawn from the same distribution as $\cS$. The second inequality follows from the fact that $\E[L(h, \bag, \alpha)] = \E[\ell(h(x), y)]$. Let $\sigma_{ij}$ be a random variable uniformly distributed in $\{-1,1\}$ and $\boldsymbol{\sigma} = (\sigma_{ij})_{i\in [n],j\in [k]}$. Then by a standard Rademacher complexity argument (e.g., \cite{mohri}), we have
\begin{align}
\label{eq:first_rad}
    \frac{1}{n}\E_{\cS, \cS'} \left[\sup_{h \in \hyps} 
\sum_{i=1}^n L(h, \bag_i' \alpha_i')- L(h, \bag_i,\alpha_i)\right]
&= \frac{1}{kn}\E_{S, S'}\left[ \sup_{h\in \hyps} \sum_{i,j} \tL(h(x'_{ij}, \alpha_i') - \tL(h(x_{ij}), \alpha_i) \right] \nonumber \\
&= \frac{2}{kn}\E_{S, \boldsymbol{\sigma}}\left[\sup_{h\in \hyps} \sum_{i,j} \sigma_{ij}\tL(h(x_{ij}), \alpha_i) \right]~.
\end{align}
Let $\beta > 0$ and $\eta = \frac{\beta}{kn}$. By Proposition~\ref{prop:truncated} we know that we can rewrite $\tL(h(x_{ij}), \alpha_i)$ as $\tL_\eta(h(x_{ij}, \alpha_i) + Z_{ijh}$. Moreover $|Z_{ijh}| \leq 2kB$ and by the union bound, with probability at least $1- \beta$, for all $i,j,h$ we have $Z_{ijh} = 0$. Using this fact we may bound  \eqref{eq:first_rad} as follows
\begin{align}
\frac{2}{kn}\E_{S, \boldsymbol{\sigma}}\left[\sup_{h\in \hyps} \sum_{i,j} \sigma_{ij}\tL(h(x_{ij}), \alpha_i) \right] &\leq 
\frac{2}{kn} \E_{\cS, \boldsymbol{\sigma}}\left[
\sum_{i,j} \sigma_{ij}\tL_\eta(h(x_{ij}), \alpha_i)\right] + 
\frac{2}{kn} \E_{\cS, \boldsymbol{\sigma}}\left[\sup_{h \in \hyps}
\sum_{i,j} \sigma_{ij} Z_{ijh} \right] \nonumber\\
&\leq \frac{2}{kn} \E_{\cS, \boldsymbol{\sigma}}\left[
\sum_{i,j} \sigma_{ij}\tL_\eta(h(x_{ij}), \alpha_i)\right] + 4 k B \beta~. \label{eq:rad2}
\end{align}
Finally, we use the definition of $\tL_\eta$ to bound the above expression:
\begin{align*}
   \tL_\eta(h(x_{ij}, \alpha_i) 
   &= \left(k(\alpha -p) + p \right)\ind_{|\alpha_i - p| < c(k, \eta)} \ell(h(x_{ij}), 1)
   + \left(k(p - \alpha_i) + (1 - p) \right)\ind_{|\alpha_i - p| < c(k, \eta)} \ell(h(x_{ij}), 0) \\
   &= \psi^{(1)}_{i} \big(\ell(h(x_{ij},1)\big) + 
   \psi^{(0)}_{i}\big(\ell(h(x_{ij}, 1)\big)
\end{align*}
where $\psi^{(1)}_{i}(z) = \left(k(\alpha - p) + p\right) \ind_{|\alpha_i - p| < c(k,\eta)}$ and $\psi^{(0)}_i$ is similarly defined. Notice that due to the indicator function we must have that $\psi_i^{(r)}$ is $(k c(k,\eta) + 1)$-Lipchitz. Define the set of functions $\hyps_\ell^{(1)} = \{x \to \ell(h(x), 1) \colon h \in  \hyps\}$ and $\hyps_\ell^{(0)} = \{x \to \ell(h(x), 0) \colon h \in \hyps\}$ we then have that 
\begin{flalign*}
&\frac{2}{kn} \E_{\cS, \boldsymbol{\sigma}}\left[
\sum_{i,j} \sigma_{ij}\tL_\eta(h(x_{ij}), \alpha_i)\right] & \\
&\leq \frac{2}{kn} \E_{\cS, \boldsymbol{\sigma}}\left[
 \sup_{h \in \hyps}\sum_{i,j} \sigma_{ij} \psi_i^{(1)}\big(\ell(h(x_{ij}, 1)\big)
 \right]
 + \frac{2}{kn} \E_{\cS, \boldsymbol{\sigma}}\left[
 \sup_{h \in \hyps}\sum_{i,j} \sigma_{ij} \psi_i^{(0)}\big(\ell(h(x_{ij}, 0)\big)
 \right] \\
 &= \frac{2}{kn} \E_{\cS, \boldsymbol{\sigma}}\left[
 \sup_{g \in \hyps_\ell^{(1)}}\sum_{i,j} \sigma_{ij} \psi_i^{(1)} \circ g(x_{ij})
 \right]
 +  \frac{2}{kn} \E_{\cS, \boldsymbol{\sigma}}\left[
 \sup_{g \in \hyps_\ell^{(0)}}\sum_{i,j} \sigma_{ij} \psi_i^{(0)} \circ g(x_{ij})
 \right]~.
\end{flalign*}
Finally, we use the fact that all functions $\psi_i^{(r)}$ are $(k c(k,\eta) + 1)$-Lipchitz together with Talagrand's contraction Lemma \cite{talagrand, mohri} to show that the above expression is bounded by
\begin{align*}
2\big(k c(k, \eta) + 1 \big)\left(\rad_{kn}(\hyps_\ell^{(1)}) +
 \rad_{kn}(\hyps_\ell^{(0)})\right)
&= 2\left(\sqrt{k \log \frac{1}{\eta}} + 1 \right)
\left(\rad_{kn}(\hyps_\ell^{(1)}) +
 \rad_{kn}(\hyps_\ell^{(0)})\right) \\
 &=2\left(\sqrt{k \log \frac{kn}{\beta}} + 1 \right)
\left(\rad_{kn}(\hyps_\ell^{(1)}) +
 \rad_{kn}(\hyps_\ell^{(0)})\right)~.
 \end{align*}
Replacing this expression in \eqref{eq:rad2} and setting $\beta = \frac{1}{kn}$ we see that 
$$
\E[\Phi(\cS)] \leq 2\left(\sqrt{2 k \log (kn)} + 1 \right)
\left(\rad_{kn}(\hyps_\ell^{(1)}) +
 \rad_{kn}(\hyps_\ell^{(0)})\right) + \frac{4B}{n}~.
$$
 
Putting it all together we see that with probability at least $1 - \delta$ over $\cS$ we have:
 \begin{equation*}
     \Phi(\cS) \leq 2\left(\sqrt{2 k \log (kn)} + 1 \right)
\left(\rad_{kn}(\hyps_\ell^{(1)}) +
 \rad_{kn}(\hyps_\ell^{(0)})\right) + \frac{4B}{n} + 4B \sqrt{\frac{k\log(1/\delta)}{2n}}. 
 \end{equation*}
The result follows from the definition of $\Phi$. 
\end{proof}

\subsection{Proof of Theorem \ref{thm:convex_pickone}}
\allowdisplaybreaks
\thmConvexPickone*
\begin{proof}
    We focus on the sequence of parameter vectors $\bw_0, \ldots, \bw_n$, which is a stochastic process the filtration of which is defined $$\cF_t = \{ (x_{11}, y_{11}, \tilde{y}_{11}), \ldots, (x_{tj}, y_{tj}, \tilde{y}_{tj}), J_1, \ldots, J_{t-1} \}$$ such that $J_i$ is a uniform random variable from $[k]$ for any $i$. 
We introduce some short-hand notations: the unbiased gradient based on surrogate labels is denoted by $\tg_{ti} = \tg ( x_{tj}, \ty_{tj})$ and $g_t = \tg_{tJ_t}$ is the gradient that is used by Algorithm \ref{alg:pickone}, i.e., the unbiased gradient compute based on a randomly selected instance from $\bag_t$ and the corresponding surrogate label. Furthermore we will denote the instance level gradient as $g_{tj} = g_{w_t} ( x_{tj}, y_{tj})$.
We can rewrite the L2 error as follows:
\begin{align}
    \E \Big[ \| \bw_{t+1} - \bw \|^2 \vert \cF_t \Big] \notag 
    & = 
    \E \left[ \left\| \Pi_{\mathcal{W}} \Big( \bw_{t} - \eta_t \tg_{t} - \bw \Big) \right\|^2 \vert \cF_t \right] \notag \\
    & \le 
    \E \left[\left\| \bw_{t} -\eta_t \tg_t - \bw \right\|^2 \vert \cF_t \right] \notag \\
    & \le 
    \E \left[ \left\| \bw_{t} - \bw \right\|^2 \vert \cF_t  \right] - 2 \eta_t \E \left[ \langle \tg_{t}, \bw_t - \bw \rangle \vert \cF_t \right] + \eta_t^2 \E \left[  \left\| \tg_{t} \right\|^2\vert \cF_t \right] \notag \\
    & = 
    \E \left[ \left\| \bw_{t} - \bw \right\|^2 \vert \cF_t  \right] - 2 \eta_t \frac{1}{k} \sum_{j=1}^k \E \left[ \langle g_{tj}, \bw_t - \bw \rangle \vert \cF_t \right] + \eta_t^2 \E \left[  \left\| \tg_{t} \right\|^2\vert \cF_t \right] \label{eq:ogd_decomp}
\end{align}
where in the last step follows from Corollary \ref{cor:debiasedLoss} applied to each term of the summation by noting that $\bw_t -\bw $ is not a random quantity with respect to filtration $\cF_t$. As a next step, we will upper bound using the assumption that $\E[\| g_{\bw_t} ( x, y) \|^2 ] \le G^2$ for any $t$ as
\[
\E \left[  \left\| \tg_{t} \right\|^2\vert \cF_t \right]  = \E \left[  \left\| \tg_{tJ_t}  \right\|^2\vert \cF_t \right] = \frac{1}{k} \sum_{j=1}^k  \E \left[ \left\| \tg_{i,j} \right\|^2\vert \cF_t, J_t = j \right]
\]
so we shall focus on $\E \Big[ \big\| \tg_{\bw}  ( x, \ty) \big\|^2 \Big]$ assuming that $\bw$ is fixed, and $\tY \sim \text{Bernoulli} ( \alpha_t )$. We can compute an upper bound as

\begin{align*}
\E \Big[ \big\| \tg_{\bw}  ( x, \ty) \big\|^2 \Big]
& = 
k^2 \E \left[\left\| g_{\bw}  ( x, \ty) \right\|^2 \right] \\ 
&\qquad -
2 k(k-1) \E \left[ \langle g_{\bw}  ( x, \ty), p \nabla g_{\bw}  ( x, 1) + (1 - p)\nabla g_{\bw}  ( x, 0) \rangle \right] \\
&\qquad + 
(k-1)^2 \E \left[ \left\| p g_{\bw}  ( x, 1) + (1 - p)\nabla g_{\bw}  ( x, 0) \right\|^2 \right] \\
& \le 
k^2 G^2 - 2 k(k-1)p \E \left[ \langle g_{\bw}  ( x, \ty), g_{\bw}  ( x, 1)\rangle \right]  \\
&\qquad - 2 k(k-1)(1-p) \E \left[ \langle g_{\bw}  ( x, \ty), \nabla g_{\bw}  ( x, 1)\rangle \right] \\
&\qquad + 
\underbrace{(k-1)^2 p^2 \E \left[\left\| g_{\bw}  ( x, 1) \right\|^2 \right] + (k-1)^2 (1-p)^2 \E \left[\left\| g_{\bw}  ( x, 0) \right\|^2 \right]}_{\le 2k^2G^2} \\
&\qquad + 
(k-1)^2 p ( 1-p) \E \left[ \langle g_{\bw}  ( x, 1), g_{\bw}  ( x, 0) \rangle\right] \\
& \le     
5k^2 G^2 
\end{align*}
where the last inequality follows from the Cauchy-Schwartz inequality applied to 
\[ \vert\langle g_{\bw}  ( x, 1), g_{\bw}  ( x, 0)\rangle \vert \le \| g_{\bw}  ( x, 1)\| \cdot \|g_{\bw}  ( x, 0)\| \le G^2
\]
and the fact that $p (1-p) \le 1/4$.
The convexity of the loss and (\ref{eq:ogd_decomp}) yield that
\begin{align*}
    \E \left[ F( \bw_t) - F( \bw ) \vert \cF_t \right] & \le 
    \frac{1}{k} \sum_{j=1}^k \E \Big[ \langle g_{tj}, \bw_t - \bw \rangle \vert \cF_t \Big] \\
    & \le 
    \frac{ \E \left[ \left\| \bw_{t} - \bw \right\|^2 \vert \cF_t  \right] - \E \Big[ \| \bw_{t+1}- \bw \|^2 \vert \cF_t \Big]}{2 \eta_t} + \frac{\eta_t}{2} \E \left[  \left\| \tg_{t} \right\|^2\vert \cF_t \right]  \\
    & \le 
    \frac{ \E \left[ \left\| \bw_{t} - \bw \right\|^2 \vert \cF_t  \right] - \E \Big[ \| \bw_{t+1}- \bw \|^2 \vert \cF_t \Big]}{2 \eta_t} + \frac{5 \eta_t k^2 G^2}{2} 
\end{align*}
which can be applied recursively as
\begin{align*}
    \E \left[\sum_{t=n-s}^n (F( \bw_t) - F( \bw )) \Big\vert \cF_t \right] & \le \frac{1}{2\eta_{n-s}} \E \left[ \left\| \bw_{n-s} - \bw \right\|^2 \vert \cF_{n-s}  \right] \\& \quad + \sum_{t=n-s+1}^n \frac{\E \left[ \left\| \bw_{t} - \bw \right\|^2 \vert \cF_{t}  \right]}{2}\left(\frac{1}{\eta_t} - \frac{1}{\eta_{t-1}}\right) + \frac{5G^2k^2}{2} \sum_{t=n-s}^n \eta_t~.
\end{align*}
Now we can follow Theorem 2 of~\cite{Shamir013} with $\eta = c/(k\cdot \sqrt{t})$ where the index $t$ is meant over bags. Let us upper bound $\E [ \| \bw_t - \bw \|^2 ]$ by $D^2$ and pick $\bw = \bw_{n-s}$ which yields
\begin{align*}
    \E \left[\sum_{t=n-s}^n (F( \bw_t) - F( \bw_{n-s} )) \Big\vert \cF_t \right] & \le 
    \frac{kD^2}{2c}\Big(\sqrt{n} - \sqrt{n-s-1} \Big) + \frac{5 c k G^2}{2} \sum_{t=n-s}^n \frac{1}{\sqrt{t}} \\
    & \le \left(\frac{kD^2}{2c} + 5 c k G^2 \right) \Big(\sqrt{n} - \sqrt{n-s-1} \Big) \\
    & \le \left(\frac{kD^2}{2c} + 5 c k G^2 \right) \frac{s+1}{\Big(\sqrt{n} + \sqrt{n-s-1} \Big)} \\
    &\le \left(\frac{kD^2}{2c} + 5 c k G^2 \right) \frac{s+1}{\sqrt{T}}
\end{align*}
where we used the fact that $\sum_{t=n-s}^n 1/\sqrt{t} \le 2( \sqrt{n} - \sqrt{n-s-1} )$. The rest of the proof is analogous to the proof of Theorem 2 of~\cite{Shamir013}.

The second part of Theorem follows a very similar pattern, where we apply Theorem \ref{t:gradient_squared} to upper bound $\E \left[ \| \tg( x_i, \alpha_t) \|^2 \right]$.
\end{proof}

\subsection{Proof of Theorem~\ref{thm:multiclass}}
\thmMulticlass*
\begin{proof}
Fix $i \in [k]$ and define $\alpha_c^{(i)} = \alpha_c - \frac{1}{k} \ind_{y_i = c}$
so that $\ind_{y_i = c} = k(\alpha_c - \alpha_c^{(i)})$. 
Moreover, since $\alpha_c^{(i)} = \frac{1}{k} \sum_{j \neq i} \ind_{y_j = c}$ and the samples $(x_1, y_1), \dots, (x_k, y_k)$ are independent, we are guaranteed that $\alpha^{(i)}_c$ is independent of $x_i$.
Let $x_i \in \bag$ and let $y_i$ be its corresponding label we then have
\begin{align*}
    \E[g(x_i, y_i)] &= \sum_{c=1}^C \E[\ind_{y_i = c}\,g(x_i, c)] 
    = \sum_{c=1}^C \E[k(\alpha_c - \alpha_c^{(i)} ) g(x_i, c)] \\
    &= \sum_{c=1}^C\E[k\alpha_c\, g(x_i, c)]  - \sum_{c=1}^C \E[k\alpha_c^{(i)}\,g(x_i, c)]~. 
\end{align*}
Now, since $\alpha_c^{(i)}$ is independent of $g(x_i,c)$ and $\E[\alpha_c^{(i)}] = \frac{k-1}{k}p_c$, we have that $\E[k\alpha_c^{(i)}g(x_i, c)] = (k-1)p_c \E[g(x_i, c)]$. Therefore the above expression can be simplified as
\begin{align*}
    \E_{x, y \sim \cD}[g(x, y)] = \E[g(x_i, y_i)] &=\sum_{c=1}^C\E[k \alpha_c g(x_i, c)]  - \sum_{c=1}^C (k-1)p_c\E[g(x_i, c)]\\
    & = \E\left[\sum_{c=1}^C \big(k \alpha_c - (k-1)p_c\big)g(x_i, c)\right] = \E[\tg(x_i, \boldsymbol{\alpha})]~,
\end{align*}
as claimed.
\end{proof}

\ \\ 
\subsection{Further plots}
\label{app:cifar10}
\begin{figure}[ht!]
    \centering
    \includegraphics[width=0.29\textwidth]{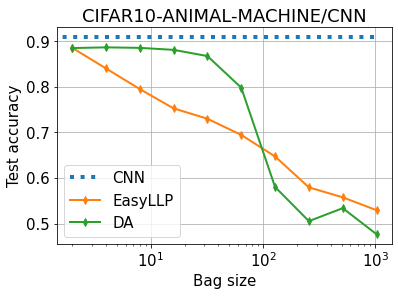}\quad
    \includegraphics[width=0.29\textwidth]{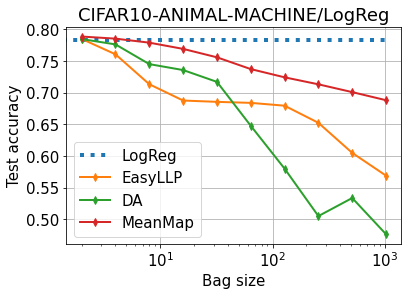}
    \caption{Performance of various methods trained with label proportion using CIFAR-10 on specific classes. Here, the event level performance is given by either a logistic regression algorithm (``LogReg") or a CNN. ``DA" denotes the method from \cite{dulac2019deep}. }
\end{figure}

\end{document}